\documentclass{article}
\usepackage{ijcai18}  

\usepackage{times}
\usepackage{times}
\usepackage{xcolor}
\usepackage{soul}
\usepackage[utf8]{inputenc}
\usepackage[small]{caption}

\usepackage[subtle]{savetrees}

\usepackage{listings}
\lstdefinestyle{asp}{basicstyle=\ttfamily\scriptsize\upshape}

\lstdefinestyle{pseudo}{basicstyle=\ttfamily\footnotesize\upshape}
\usepackage{url}  
\usepackage{graphicx}  

\usepackage{amssymb}
\usepackage{amsmath}
\usepackage{amsthm}
\usepackage{color}
\usepackage[toc,page]{appendix}

\newtheorem{theorem}{Theorem}

\newtheorem{definition}{Definition}

\newtheorem*{example*}{Example}

\newcommand{\comment}[1]{}

\newcommand{\smallmath}[1]{\mbox{\scriptsize ${#1}$}}

\newcommand{\tinymath}[1]{\mbox{\tiny ${#1}$}}

\newcommand{\cnL}{\mbox{${\cal L}$}}            

\newcommand{\cnSs}{\mbox{${\cal S}$}}           
\newcommand{\cnAs}{\mbox{${\cal A}$}}           
\newcommand{\cnTs}{\mbox{${\cal T}$}}           
\newcommand{\cnOs}{\mbox{${\cal O}$}}           
\newcommand{\cnCs}{\mbox{${\cal C}$}}           
\newcommand{\cnPSs}{\mbox{${\cal P}$}}          
\newcommand{\cnUs}{\mbox{${\cal U}$}}           

\newcommand{\cnN}{\mbox{${N}$}}                 
\newcommand{\cnPs}{\mbox{$\Pi$}}                
\newcommand{\cnPUOs}{\mbox{${\Gamma}$}}         

\newcommand{\cnPO}{\mbox{$\lambda$}}            
\newcommand{\cnOUO}{\mbox{$\tau$}}              
\newcommand{\cnPUO}{\mbox{$\gamma$}}            
\newcommand{\cnTLO}{\mbox{$\mu$}}               
\newcommand{\cnUUF}{\mbox{${\delta}$}}          

\newcommand{\cnS}{\mbox{$s$}}                   
\newcommand{\cnP}{\mbox{$\pi$}}                 

\newcommand{\cnPS}{\mbox{$p$}}                  

\newcommand{\cnCS}{\cnS^{c}}                    
\newcommand{\cnPUS}{\cnS^{\smallmath{\cnPUO}}}       
\newcommand{\cnCUS}{\cnS^{\smallmath{\cnOUO}}}       

\newcommand{\cnH}{\mbox{$H$}}                   

\newcommand{\cnNs}{\mbox{${\cal N}$}}                  
\newcommand{\cnFPs}{\mbox{${F}$}}                      

\newcommand{\cnAN}{\mbox{$Q$}}                         
\newcommand{\cnAH}{\mbox{$\cal X$}}                    
\newcommand{\cnANs}{\mbox{$\cal Q$}}                   

\newcommand{\cnSF}{\mbox{$\phi$}}                      
\newcommand{\cnTF}{\mbox{$\psi$}}                      
\newcommand{\cnCF}{\mbox{$\varrho$}}                   
\newcommand{\cnUTF}{\mbox{$\sigma$}}                   

\newcommand{\cnPredUpdate}{\mbox{\text{\bf PredictionUpdate}}}      
\newcommand{\cnCorrUpdate}{\mbox{\text{\bf CorrectionUpdate}}}      
\newcommand{\cnTransFuncUpdate}{\mbox{\text{\bf TransFuncUpdate}}}  
\newcommand{\cnUtilUpdate}{\mbox{\text{\bf UtilityUpdate}}}         
\newcommand{\cnActionUpdate}{\mbox{\text{\bf ActionUpdate}}}         

\newcommand{\cnUpdatePass}{\mbox{\text{\bf UpdatePass}}}       

\newcommand{\cnPredUpdateP}{\mbox{\text{\bf P}}}      
\newcommand{\cnCorrUpdateP}{\mbox{\text{\bf C}}}      
\newcommand{\cnTransFuncUpdateP}{\mbox{\text{\bf T}}}  
\newcommand{\cnUtilUpdateP}{\mbox{\text{\bf U}}}         
\newcommand{\cnActionUpdateP}{\mbox{\text{\bf A}}}        

\newcommand{\cnUpdate}{\mbox{\text{\bf Update}}}                        

\title{Online Learning and Planning in Cognitive Hierarchies}
\author{Bernhard Hengst
\and
Maurice Pagnucco
\and
\vspace{0.5mm}
David Rajaratnam\\
\vspace{1mm}
{\bf \Large Claude Sammut
\and
Michael Thielscher}\\
School of Computer Science and Engineering\\
The University of New South Wales, Australia
}

\begin{document}
\maketitle

\begin{abstract}
  Complex robot behaviour typically requires the integration of multiple robotic and Artificial Intelligence (AI) techniques and components. Integrating such disparate components into a coherent system, while also ensuring global properties and behaviours, is a significant challenge for cognitive robotics. Using a formal framework to model the interactions between components can be an important step in dealing with this challenge. In this paper we extend an existing formal framework~\cite{clark2016framework} to model complex integrated reasoning behaviours of robotic systems; from symbolic planning through to online learning of policies and transition systems. Furthermore the new framework allows for a more flexible modelling of the interactions between different reasoning components.

\end{abstract}


\section{Introduction}

Complex robot behaviour typically requires the integration of multiple robotic and Artificial Intelligence (AI) techniques and components. For example, a typical social/domestic robot must integrate vision, speech recognition, motion planning, as well as high-level control. More concretely, the object manipulation task, where a robot must manipulate objects to achieve a goal, requires the combination of high-level task planning as well as geometric reasoning and motion planning~\cite{DBLP:conf/icra/Lozano-PerezJMO87}. Integrating such disparate components into a coherent system, while also ensuring global properties and behaviours, is a significant challenge for robotics.

Prior research in this area has mostly focused on integrating components on a case-by-case basis. For example, the object manipulation challenge has been tackled by combining symbolic planning with probabilistic sampling~\cite{DBLP:journals/ijrr/CambonAG09,DBLP:conf/wafr/GarrettLK14}. However, such individual cases do not address the larger challenge of developing general principles for integration or incorporating sub-systems into larger, more complex, systems.

Dealing with the more general integration issue has largely been left as a software development challenge. Diagrammatic or descriptive modelling is performed, for example using the Robot Control System (RCS)~\cite{Albus-2001}, followed by integration using a software framework, such as ROS~\cite{Quigley09}. However, while software frameworks provide the programming tools and middleware required to build and integrate robotic components, they do not, and were not intended to, provide a structured basis on which to analyse robot design and establish global behavioural properties. As a result, while the behaviour of individual components of a system may be well-understood, the behaviour of the system as a whole is often unclear.

A solution to this problem has been to use a formal framework for integrating disparate representations into a cognitive hierarchy. The framework introduced by~\citeauthor{clark2016framework}~(\citeyear{clark2016framework}), and extended by~\citeauthor{hengst2018context}~(\citeyear{hengst2018context}), adopts a meta-theoretic approach to modelling cognitive robotic systems. It formalises the interaction between nodes in a cognitive hierarchy, while making no commitments about the details within individual nodes. It presents the possibility for a more structured approach to the integration challenge and the potential for establishing global properties of a system.

As it stands, the formal framework has been mainly applied to cases where the robot's behaviour has been largely reactive, and has only limited support for more complex deliberative reasoning, such as online learning and planning. In this paper we extend the formal meta-theory to provide the flexibility required to capture a wider variety of deliberative behaviour; from planning through to online learning of policies and transition models. Furthermore our new framework allows for a greater flexibility in propagating the utilities of actions between levels of a cognitive hierarchy.

The rest of the paper proceeds as follows. We first provide a brief overview of the notion of a cognitive hierarchy and the underlying intuitions necessary to understand the formal framework. We then present a motivating example to serve as a running example that shows how symbolic planning and reinforcement learning can be integrated into a cognitive hierarchy. Next, the formal framework will be presented and properties shown, highlighting how our extension differs from the original. We will then finish with a discussion of the formal developments and some complexity considerations, with reference to the motivating example.


\section{Cognitive Hierarchy Overview}
\label{secCogHiOverview}

We start here with a brief overview of some of the underlying ideas and intuitions of the theory. A cognitive hierarchy (Figure~\ref{figCogHiOverview}) consists of a set of nodes, where a node has two primary functions: world-modelling and behaviour-generation. World-modelling involves maintaining a \emph{belief state}, while behaviour-generation is achieved through \emph{policies}, where a policy maps states to sets of actions.

\begin{figure}[h]
  \centering
    \includegraphics[width=0.40\textwidth]{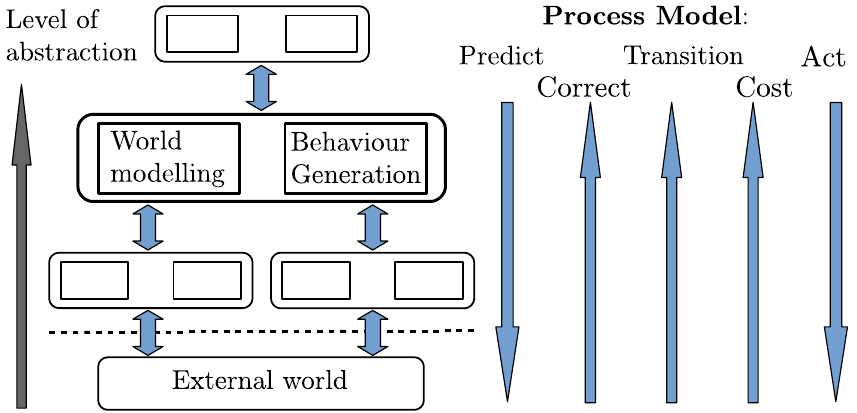}
    \caption{Overview of a Cognitive Hierarchy}
  \label{figCogHiOverview}
\end{figure}

The belief state of a high-level node represents an abstraction of the belief state of lower-level nodes, ultimately abstracting from the physical environment. These belief states are updated through sensing. Similarly, actions at a high-level successively reduce to low-level actions that drive physical actuators. The external world itself can be modelled simply as a lowest-level node, albeit a node with opaque internal behaviour that is accessed only through sensing and actions.

A \emph{process model} is required to specify how a cognitive hierarchy is updated. In the model by~\citeauthor{clark2016framework}~(\citeyear{clark2016framework}) a simple two step process was defined, consisting of a sensing pass followed by an action generation pass. Unfortunately, such a model is no longer adequate for scenarios that incorporate online learning and planning. Consequently, we introduce a multi-stage process model (Figure~\ref{figCogHiOverview}) involving a number of distinct passes up and down the cognitive hierarchy .


First, a \emph{prediction update} is performed down the hierarchy, updating how the robot predicts the external world will change. Next a \emph{correction update} progressively corrects the predicted belief states based on actual observations of the external world. Thirdly, comparing the change between observations, a \emph{transition function update} is performed to learn or refine each node's transition model. Fourthly, a \emph{utility update} is performed whereby costs of lower-level actions are passed up for use by higher-level action planners/learners. Finally, the \emph{action update} phase selects policies which results in action selection down the hierarchy.


\section{Motivating Scenario}
\label{s:example}
As an explanatory aid to formalising the integration of planning and learning in a cognitive hierarchy we consider the task of a robot that needs to reach a goal by navigating between rooms (Figure~\ref{figMotivatingExample} (a)). While we allow for stochasticity of robot motion, for example due to slipping of the wheels, for the simplicity of presentation we assume that the robot is able to accurately sense its location within the environment.

\begin{figure}[h]
  \centering
  \begin{minipage}[b]{0.34\linewidth}
    \includegraphics[width=1.0\textwidth]{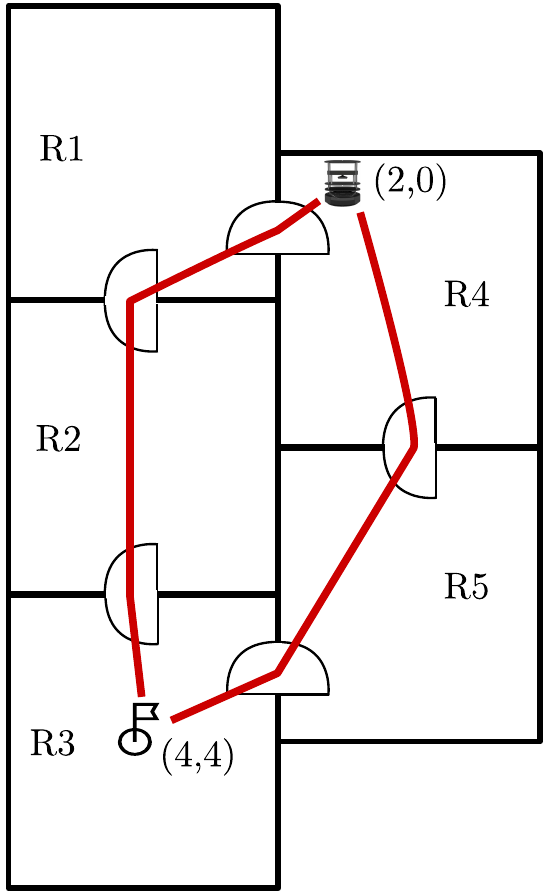}
    \centering
    (a)
  \end{minipage}
  \hspace{0.005\linewidth}
  \begin{minipage}[b]{0.38\linewidth}
    \includegraphics[width=1.0\textwidth]{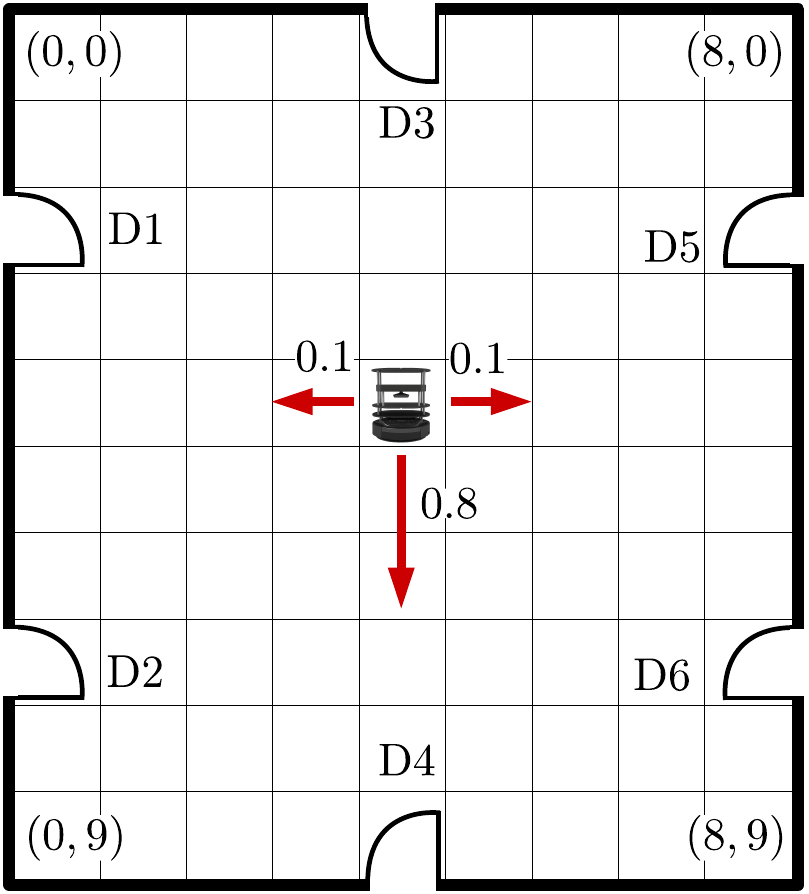}
    \centering
    (b)
  \end{minipage}
  \hspace{0.005\linewidth}
  \begin{minipage}[b]{0.22\linewidth}
    \includegraphics[width=1.0\textwidth]{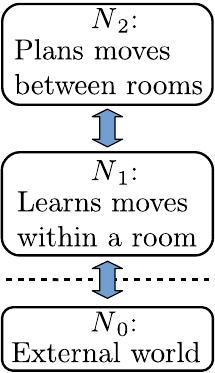}
    \centering
    (c)
  \end{minipage}

  \caption{A model for robot navigation with stochastic robot motion. Rooms are identical except that doorways may differ for individual rooms. The robot and goal are located in rooms R4 and R3 respectively.}
  \label{figMotivatingExample}
\end{figure}

This problem can be modelled as a three level cognitive hierarchy (Figure~\ref{figMotivatingExample} (c)); a lowest-level node representing the external world, a lower-level node for learning motion control and costs, and a higher-level node for planning abstract actions to move between rooms. At the planning level there are two (loop free) paths for the robot to reach the goal: R4-R5-R3 and R4-R1-R2-R3. From a global perspective path R4-R5-R3 traverses fewer doorways while path R4-R1-R2-R3 is a shorter distance. However, determining which of these is optimal requires accounting for the stochastic motion of the robot. For example, if high-precision movements are difficult then it may be preferable to minimise the doorways traversed rather than minimising the distance travelled.

To highlight the learning of the motion model we simplify the scenario to ensure that each room is identical except for the presence of the key features of the doorways and the goal location (Figure~\ref{figMotivatingExample} (b)). Consequently, at the lower-level the system learns both a transition model as well as an action model for moving to the different features within a room.

A key contribution of this paper is to show how different robotic and AI techniques can be combined in a principled manner within a single robotic system. The example illustrates this point by constructing a simple integrated system where a high-level symbolic planner is able to determine optimal paths using cost information extracted from a low-level transition model learnt through reinforcement learning.


\section{Cognitive Hierarchy Meta-Theory}

For the sake of brevity the following presentation both summarises and extends the formalisation of cognitive hierarchies as introduced in~\citeauthor{clark2016framework}~(\citeyear{clark2016framework}), and extended in~\citeauthor{hengst2018context}~(\citeyear{hengst2018context}). We shall, however, highlight how our contribution differs from these earlier works. The essence of this framework is to adopt a meta-theoretic approach, formalising the interaction between abstract cognitive nodes, while making no commitments about the representation and reasoning mechanism within individual nodes.


\subsection{Nodes}
Nodes in a cognitive hierarchy achieve goals through maintaining a belief state and a policy. The belief state can be modified through two distinct mechanisms: by sensing lower-level nodes (\emph{correction update}) or by the combination of actions and higher-level context (\emph{prediction update}). The prediction update involves applying a transition model of how the world changes, which itself can be learnt. As well as modelling observed changes in the world a robot also takes actions. Actions are selected by applying the belief state to a policy. The appropriate policy is determined by a planning (or learning) operator that selects a policy based on the transition model, action costs, and the goals to be achieved.

\begin{definition}
\label{d:cnode}
A cognitive language is a tuple $\cnL =(\cnSs, \cnAs, \cnTs, \cnOs, \cnCs, \cnPSs, \cnUs)$, where $\cnSs$ is a set of belief states, $\cnAs$ is a set of actions, $\cnTs$ is a set of task parameters, $\cnOs$ is a set of observations, and $\cnCs$ is a set of contextual elements, $\cnPSs$ is a set of internal planning states, and $\cnUs$ is a set of utilities.

A cognitive node is a tuple $\cnN = (\cnL, \cnPs, \cnPUOs, \cnPO, \cnTLO,  \cnOUO, \cnUUF, \cnPUO^0, \cnP^0, \cnS^0, \cnPS^0)$ s.t:
\begin{itemize}
\item $\cnL$ is the cognitive language for $\cnN$, with initial belief state $\cnS^0\in\cnSs$, and initial planning state $\cnPS^0\in\cnPSs$.

\item An observation update operator $\cnOUO: \cnSs \times 2^{\smallmath{\cnOs}} \rightarrow \cnSs$.

\item $\cnUUF$ is a utility update function such that $\cnUUF : \cnPSs  \times 2^{\smallmath{\cnUs}}\rightarrow \cnPSs$.

\item $\cnPUOs$ is a set of transition functions such that for all $\cnPUO \in \cnPUOs$, $\cnPUO: \cnSs \times 2^{\smallmath{\cnCs}} \times 2^{\smallmath{\cnAs}} \rightarrow \cnSs$ with initial transition function $\cnPUO^0 \in \cnPUOs$.

\item A transition learning operator\\$\cnTLO: \cnPUOs \times \cnSs \times 2^{\smallmath{\cnCs}} \times 2^{\smallmath{\cnAs}} \times \cnSs \rightarrow \cnPUOs$.

\item $\cnPs$ is a set of policies such that for all $\cnP \in \cnPs$, $\cnP : \cnSs \rightarrow 2^{\smallmath{\cnAs}}$ and initial policy $\cnP^0 \in \cnPs$.

\item A planning operator $\cnPO\!:\!\cnPs\!\times\!\cnPSs\!\times\!\cnPUOs\!\times\!2^{\smallmath{\cnTs}}\!\times\!\cnSs \rightarrow \cnPs\!\times\!\cnPSs$.

\end{itemize}
\end{definition}

Definition~\ref{d:cnode} differs from the original in two important respects. First, it allows for learning of the transition function through the use of a \emph{transition learning operator}. Secondly, it allows for planning (or learning of an action policy) through the use of a richer function for policy selection.

An \emph{internal planning state} provides a mechanism to model complex planning behaviour operating over multiple calls to the planning function. For example, for a given belief state, the planning function might initially return a do-nothing action, when the underlying planner has not had time to find a plan, but return a plan on subsequent calls. The planning operator also incorporates utilities, where determining the costs of individual actions requires knowing the cost of executing the underlying low-level actions. A \emph{utility update function} incorporates these costs into the planner's internal state.



The motivating example can now be instantiated with low and high level nodes. The low-level node is encoded as a simple Q-learner~\cite{Watkins1992} that learns the transition function and policy to navigate within a room.

\newcommand{\exNorth}{\mbox{${N}$}}
\newcommand{\exSouth}{\mbox{${S}$}}
\newcommand{\exWest}{\mbox{${W}$}}
\newcommand{\exEast}{\mbox{${E}$}}

\newcommand{\exLogicText}[1]{\mbox{\text{\ttfamily\footnotesize\upshape{#1}}}}

\newcommand{\exLocation}{\exLogicText{at}}
\newcommand{\exGoal}{\exLogicText{goal}}
\newcommand{\exUnknown}{\exLogicText{unkn}}
\newcommand{\exDoor}[1]{\exLogicText{d#1}}
\newcommand{\exRoom}[1]{\exLogicText{r#1}}
\newcommand{\exMoveG}{\exLogicText{mv\_goal}}
\newcommand{\exMoveR}{\exLogicText{trv}}
\newcommand{\exCTF}{\exLogicText{ctf}}
\newcommand{\exCBF}{\exLogicText{cbf}}

\begin{example*}
 For the sub-symbolic node $\cnN_1$, the belief state $\cnSs_1$ consist of the robot's room location $R_1=\{\exRoom{1},\ldots,\exRoom{5}\}$ and grid $G_1=\{0,\ldots,9\} \times \{0,\ldots,10\}$, i.e. $\cnSs_1 = R_1 \times G_1$. The model of a robot with perfect sensing simplifies the set of observations, so $\cnOs_1=\cnSs_1$, and the observation update operator $\cnOUO$ simply replaces the current belief state with the observation.

 The robot can perform four distinct actions $\cnAs_1=\{\exNorth, \exSouth, \exEast, \exWest\}$ intending to move one step north, south, east, and west respectively. The task parameters $\cnTs_1=\{\exDoor{1},\ldots,\exDoor{6},\exGoal\}$ encode the intra-room navigation destinations, in particular the goal and the six numbered doors~(Figure~\ref{figMotivatingExample}). We do not require context from the higher-level node so $\cnCs_1=\{\}$, and since the lower level node $\cnN_0$ is the external world there are no utilities being supplied from below so $\cnUs_1=\{\}$. Wheel slippage is simulated by $\cnN_0$ moving the robot $80\%$ in the intended direction and $20\%$ randomly to either side. 
 
Since all rooms are similar, the $\cnN_1$ transition and planning functions use a projected state space $\{\langle\cdot,g_1\rangle | g_1 \in G_1\}$ represented by just grid locations for a typical relativised room. The transition function $\cnPUO_1$ keeps a tally of how many times a successor state is reached from a previous state and action. The planning state $\cnPS_1$ comprises the reinforcement learning $Q$-values, one set for each task parameter. The planning operator $\cnPO_1$ uses the latest transition function to update the $Q$-values in a value iteration process known as temporal difference learning. Each primitive action is assumed to incur a unit cost. The greedy policy $\cnP_1(\langle\cdot,g_1\rangle)=argmax_a Q$-value$(\langle\cdot,g_1\rangle,\;a,\;t_1)$, $t_1\in \cnTs_1$.
\end{example*}

In contrast to the low-level's reinforcement learner, the high-level node is instantiated with a symbolic planner. We use Answer Set Programming (ASP)~\cite{gelfon:answer} to represent the planning problem, and generate solutions using the ASP solver Clingo~\cite{gebser:potass}.
The solver returns models that represent solutions to the planning problem.

\begin{example*}
  The belief state of the symbolic node $\cnN_2$, consists of the robot's room location and whether it is at a specific feature; let $F=\{\exDoor{1},\ldots,\exDoor{6},\exGoal\}$ then $\cnSs_2=\{\exLocation(r, f) \mid r\in\{\exRoom{1},\ldots,\exRoom{5}\}, f \in F \cup \{\exUnknown\}\}$, where $\exUnknown$ denotes that the robot is not at a known feature. Observations simply replace the belief state so $\cnOs_2=\cnSs_2$. The actions consist of either traversing a doorway in a room or moving to the goal location, $\cnAs_2=\{\exMoveR(d) \mid d \in\{\exDoor{1},\ldots,\exDoor{6}\}\} \cup \{\exMoveG\}$.


The rooms and doorways are fixed so, in the following ASP fragment, we consider a fixed transition function $\cnPUO^0_2$ to encode the movements between rooms and the goal (Figure~\ref{figMotivatingExample}):

\begin{lstlisting}[style=asp,numbers=none]
  time(0..10).                 % 10 step time horizon
  goal_in(r3).                 % goal is in R3
  conn(r1,d4,r2,d3).           % R1-D4 connects to R2-D3
  conn(r1,d6,r4,d1). ... conn(r5,d2,r3,d5).

  % precondition axioms
  poss(mv_goal,T) :- at(R,L,T), L != goal, goal_in(R).
  poss(trv(D1)),T) :- at(R1,_,T), conn(R1,D1,R2,D2).

  % successor state axioms
  at(R,goal,T+1) :- robot_at(R,_,T), do(mv_goal,T).
  at(R2,D2,T+1) :-
      do(trv(D1),T), in(R1,_,T), conn(R1,D1,R2,D2).
\end{lstlisting}

A plan requires a sequence of actions that reaches the goal:

\begin{lstlisting}[style=asp,numbers=none]
  1 { do(A,T):poss(A,T)} 1 :- time(T), not done(T).
  done(T) :- at(_,goal,T).
  done(T+1) :- done(T), time(T).
  done :- done(T).
  :- not done.
\end{lstlisting}

Plan generation needs to take into account action costs. The utility for the node consists of a pair of functions; the first returns the cost of navigating from the robot's current location in a room to features in that room, and the second returns the cost of navigating between features within a room; $\cnUs_2=\{\langle \exCTF : F \rightarrow \mathbb{I}, \exCBF : F \times F \rightarrow \mathbb{I}\rangle\}$. The sum of the cost of each action can then be minimised:

\begin{lstlisting}[style=asp,numbers=none]
  cost(trv(D),T,C) :- at(R,unkn,T), do(trv(D),T), ctf(D,C).
  cost(mv_goal,T,C) :-
      at(R,unkn,T), do(mv_goal,T), ctf(goal,C).
  cost(traverse(D),T,C) :-
      at(R,F,T), F != unkn, do(trv(D),T), cbf(F,D,C).
  cost(mv_goal,T,C) :-
      at(R,F,T), F != unkn, do(mv_goal,T), cbf(F,goal,C).

  #minimize { C,A,T : cost(A,T,C) }.  % minimise plan cost
\end{lstlisting}


\end{example*}


\subsection{Cognitive Hierarchy}

Nodes are interlinked in a hierarchy, where sensing data and utilities are passed up the \emph{abstraction hierarchy}, while actions and context are passed down the hierarchy (Figure~\ref{figNodeInteractions}).

\begin{figure}[tb]
	\centering
	\includegraphics[width=0.45\textwidth]{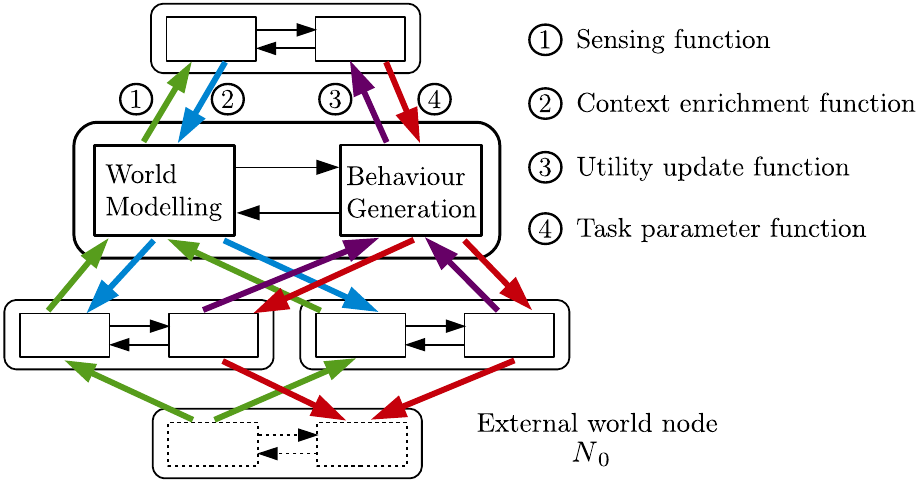}
	\caption{A cognitive hierarchy, highlighting the flow of sensing, utilities, action, and context data. Except for sensing and acting the behaviour of $\cnN_0$ is opaque.}
        \label{figNodeInteractions}
\end{figure}

\begin{definition}
\label{d:chierarchy}
A cognitive hierarchy is a tuple $\cnH = (\cnNs, \cnN_0, \cnFPs)$ s.t:
\begin{itemize}
\item $\cnNs$ is a set of cognitive nodes and $\cnN_0\in \cnNs$ is a distinguished node corresponding to the external world.
\item $\cnFPs$ is a set of function tuples $\langle \cnSF_{i,j}, \cnCF_{j,i}, \cnUTF_{i,j}, \cnTF_{j,i} \rangle \in \cnFPs$ that connect nodes $\cnN_i,\cnN_j\in\cnNs$ where:
  \begin{itemize}
    \item $\cnSF_{i,j}: \cnSs_i \rightarrow 2^{\smallmath{\cnOs_j}}$ is a sensing function, and
    \item $\cnCF_{j,i}: \cnSs_j \rightarrow 2^{\smallmath{\cnCs_i}}$ is a context enrichment function.
    \item $\cnUTF_{i,j}: \cnPSs_i \rightarrow 2^{\smallmath{\cnUs_j}}$ is a utility enrichment function.
    \item $\cnTF_{j,i}: 2^{\smallmath{\cnAs_j}} \rightarrow 2^{\smallmath{\cnTs_i}}$ is a task parameter function.
  \end{itemize}
\item Upward graph: each $\cnSF_{i,j}$ represents an edge from node $\cnN_i$ to $\cnN_j$ and forms a directed acyclic graph (DAG) with $\cnN_0$ as the unique source node of the graph.
\item Downward graph: the set of task parameter functions forms a converse to the upward graph such that $\cnN_0$ is the unique sink node of the graph.
\end{itemize}
\end{definition}

Definition~\ref{d:chierarchy} differs from~\citeauthor{hengst2018context}~(\citeyear{hengst2018context}) with the introduction of the \emph{utility enrichment} functions and the naming of the graphs (previously sensing and prediction graphs).
These functions provide the connection between nodes. For example the sensing function extracts observations from the belief state of a lower-level node, while the utility function allows the cost of actions from a lower-level node to be passed to the planner of a higher-level node.

The running example can now be encoded formally as a cognitive hierarchy, again with the following showing only the salient features of the encoding.

\begin{example*}
  Let $\cnH=(\cnNs,\cnN_0,\cnFPs)$ where $\cnNs = \{\cnN_0, \cnN_1, \cnN_2\}$ and $\cnFPs=\{\langle\cnSF_{0,1},\cnCF_{1,0},\cnUTF_{0,1},\cnTF_{1,0}\rangle, \langle\cnSF_{1,2},\cnCF_{2,1},\cnUTF_{1,2},\cnTF_{2,1}\rangle\}$ s.t.:\\
  \begin{array}[t]{l@{}l@{}l@{}l}
    &\multicolumn{3}{l}{\cnSF_{0,1}(\cnS_0 \in \cnSs_0) = \{\langle r,x,y\rangle\}}\text{, for robot location $\langle r,x,y\rangle$.}\\
    ~~&~~\cnSF_{1,2}:&~& \langle r,x,y\rangle \mapsto \{\langle r,f \rangle\}\text{, where $f$ is the feature at}\\
    &&\multicolumn{2}{l}{\text{location $x,y$, or $\exUnknown$ if there is no such feature.}}\\

    &\multicolumn{3}{l}{\cnUTF_{1,2}(\cnPS\in\cnPSs) = \{\langle \exCTF,\exCBF \rangle\}\text{, where $\exCTF$ and $\exCBF$ are}}\\
    &&\multicolumn{2}{l}{\text{generated from the Q-values of node $\cnN_1$.}}\\

    ~~&~~\cnTF_{2,1}:&~& \{\exMoveR({d})\} \mapsto \{d\}\text{, where $d$ is the given doorway.}\\
    &            & & \{\exMoveG\} \mapsto \{\exGoal\}\\

    ~~&~~\cnTF_{1,0}:&~& \{\exNorth\} \mapsto \{\langle 0,-1\rangle\}, \{\exSouth\} \mapsto \{\langle 0,1\rangle\}\\
    &            &~& \{\exWest\} \mapsto \{\langle -1,0\rangle\}, \{\exEast\} \mapsto \{\langle 1,0\rangle\}\text{, where}\\
    &&\multicolumn{2}{l}{\text{$\langle \{-1,0,1\},\{-1,0,1\}\rangle$ can be understood as}}\\
    &&\multicolumn{2}{l}{\text{signals to drive a robot base.}}\\

    &\multicolumn{3}{l}{\text{For $\cnCF_{1,0},\cnCF_{2,1},\cnUTF_{0,1}$ the functions return the empty set.}}
  \end{array}
\end{example*}

\subsection{Active Cognitive Hierarchy}

In order to formalise the behaviour of an operational cognitive system, additional definitions are required. The following are modified from the original in order to add support for learning and the requirements of a more complex process model.

\begin{definition}
\label{d:activenode}
An active cognitive node is a tuple $\cnAN = (\cnN, \cnPUO, \cnP, \cnCS, \cnPUS, \cnCUS, \cnPS)$ where: 1) $\cnN$ is a cognitive node with $\cnSs$, $\cnPs$, $\cnPUOs$, and $\cnPSs$ being its set of belief states, set of policies, set of transition functions, and internal planning states respectively, 2) $\cnCS, \cnPUS, \cnCUS \in \cnSs$ is, respectively, the current belief states, and two intermediate belief states, 3) $\cnPUO \in \cnPUOs$ and $\cnP \in \cnPs$ is the current transition function and policy respectively, and $\cnPS \in \cnPSs$ is the current internal planning state.
\end{definition}

Essentially an active cognitive node couples a (static) cognitive node with some dynamic information. A similar coupling is also performed for the hierarchy itself.

\begin{definition}
\label{d:activehierarchy}
An active cognitive hierarchy is a tuple $ \cnAH = (\cnH, \cnANs)$ where $\cnH$ is a cognitive hierarchy with set of cognitive nodes $\cnNs$ such that for each $\cnN \in \cnNs$ there is a corresponding active cognitive node $\cnAN=(\cnN, \cnPUO, \cnP, \cnCS, \cnPUS, \cnCUS, \cnPS) \in \cnANs$ and vice-versa.
\end{definition}

The active cognitive hierarchy captures the dynamic state of a system at a particular instance in time. Finally, an \emph{initial active cognitive hierarchy} is an active hierarchy where each node is initialised with the initial transition function $\cnPUO^0$, policy $\cnP^0$, and belief state $\cnS^0$ (for $\cnCS$, $\cnPUS$ and $\cnCUS$) of the corresponding cognitive node.

\subsection{Cognitive Process Model}
\label{s:processmodel}

The \emph{process model} defines how an active cognitive hierarchy evolves over time. The process model from~\citeauthor{hengst2018context}~(\citeyear{hengst2018context}) consists of two steps, a sensing update followed by an action update. Unfortunately this model is no longer adequate for the requirements of learning and planning, and we now formally define the process model consisting of the five distinction passes outlined in Section~\ref{secCogHiOverview}.

To define the process model we first define the updating of a single node, and then define the updating of the entire hierarchy. The intuition behind a node update function is that given a cognitive hierarchy and a node, the function returns an identical hierarchy except for the updated node. We define five functions to match the five passes of the hierarchy.

\begin{definition}
\label{d:predictionUpdatePre}
Let $\cnAH\!=\!(\cnH, \cnANs)$ be an active cognitive hierarchy with $\cnH\!=\!(\cnNs, \cnN_0, \cnFPs)$. The prediction update of $\cnAH$ with respect to an active cognitive node $\cnAN_i\!=\!(\cnN_i, \cnPUO_i, \cnP_i, \cnCS_i, \cnPUS_i, \cnCUS_i, \cnPS_i) \in \cnANs$, written as $\cnPredUpdate(\cnAH,\cnAN_i)$ is $\cnAH$ if $i=0$, else is an active cognitive hierarchy $\cnAH' = (\cnH, \cnANs')$ where
$\cnANs' = \cnANs\!\setminus\!\{\cnAN_i\} \cup \{\cnAN_i'\}$ and $\cnAN'_i = (\cnN_i,  \cnPUO_i, \cnP_i, \cnCS_i, \cnS, \cnS, \cnPS_i)$ s.t:\\
$\begin{array}[t]{l@{}l@{}l@{}l@{}l}
   &    \cnS &~=~& \multicolumn{2}{l}{ \cnPUO_i(\cnCS_i, C, \cnP_i(\cnCS_i)),}\\
   &       C &~=~& \bigcup~\{ & \cnCF_{x,i}(\cnPUS_x)~|~\langle \cnSF_{i,x}, \cnCF_{x,i}, \cnUTF_{i,x}, \cnTF_{x,i}  \rangle\in\cnFPs \mbox{~for~}\\
   &         &   & & ~\cnAN_x =(\cnN_x, \cnPUO_x, \cnP_x, \cnCS_x, \cnPUS_x, \cnCUS_x, \cnPS_x) \in \cnANs\}\\
\end{array}$
\end{definition}

Definition~\ref{d:predictionUpdatePre} performs a prediction update on a node, applying the transition function to the current belief state, the context, and actions. Note, since we have no access to the external world other than sensing and undertaking physical actions, the update function produces no change if the selected node is $\cnN_0$. This will similarly be the case for the transition learning, and utility update functions.

\begin{definition}
\label{d:correctionUpdatePre}
Let $\cnAH\!=\!(\cnH, \cnANs)$ be an active cognitive hierarchy with $\cnH\!=\!(\cnNs, \cnN_0, \cnFPs)$. The correction update of $\cnAH$ with respect to an active cognitive node $\cnAN_i\!=\!(\cnN_i,  \cnPUO_i, \cnP_i, \cnCS_i, \cnPUS_x, \cnCUS_x, \cnPS_i) \in \cnANs$, written as $\cnCorrUpdate(\cnAH,\cnAN_i)$ is:
\begin{itemize}
\item $\cnAH$, if $i = 0$ or there does not exist a node $\cnN_x$ such that $\langle \cnSF_{x,i}, \cnCF_{i,x}, \cnUTF_{x,i}, \cnTF_{i,x} \rangle\!\in\! \cnFPs$
\item an active cognitive hierarchy $\cnAH' = (\cnH, \cnANs')$ where\\ $\cnANs'=\cnANs\setminus\{\cnAN_i\}
\cup \{\cnAN_i'\}$ and\\ $\cnAN'_i = (\cnN_i,  \cnPUO_i, \cnP_i, \cnCS_i, \cnPUS_x, \cnOUO_i(O, \cnCUS_i), \cnPS_i)$ s.t:\\
$\begin{array}[t]{l@{}l@{}l@{}l@{}l}
  ~~~&  O =~& \bigcup  \{ & \cnSF_{x,i}(\cnCUS_x)~|~\langle \cnSF_{x,i}, \cnCF_{i,x}, \cnUTF_{x,i}, \cnTF_{i,x} \rangle\!\in\! \cnFPs \mbox{~for~}\\
   &      &             & ~\cnAN_x =(\cnN_x, \cnPUO_i, \cnP_i, \cnCS_x, \cnPUS_x, \cnCUS, \cnPS_x) \in \cnANs\}\\
\end{array}$
\end{itemize}
\end{definition}

Definition~\ref{d:correctionUpdatePre} establishes how to update the current belief state of a node based on sensing observations.

\begin{definition}
\label{d:transFuncUpdatePre}
Let $\cnAH\!=\!(\cnH, \cnANs)$ be an active cognitive hierarchy with $\cnH\!=\!(\cnNs, \cnN_0, \cnFPs)$. The update of the transition function of $\cnAH$ with respect to an active cognitive node $\cnAN_i\!=\!(\cnN_i,  \cnPUO_i, \cnP_i, \cnCS_i, \cnPUS_i, \cnCUS_i, \cnPS_i) \in \cnANs$, written as $\cnTransFuncUpdate(\cnAH,\cnAN_i)$ is $\cnAH$, if $i = 0$ else is an active cognitive hierarchy $\cnAH' = (\cnH, \cnANs')$ where $\cnANs' = \cnANs\!\setminus\!\{\cnAN_i\}
\cup \{\cnAN_i'\}$ and $\cnAN'_i = (\cnN_i,  \cnTLO(\cnPUO_i, \cnCS_i, C, \cnP_i(\cnCS_i), \cnCUS_i), \cnP_i, \cnCS_i, \cnPUS_i, \cnCUS_i, \cnPS_i)$ s.t:\\
$\begin{array}[t]{l@{}l@{}l@{}l@{}l}
   &       C &~=~ \bigcup  \{ & \cnCF_{x,i}(\cnPUS_x)~|~\langle \cnSF_{i,x}, \cnCF_{x,i}, \cnUTF_{i,x},\cnTF_{x,i} \rangle\in\cnFPs \mbox{~for~}\\
   &         &             & ~\cnAN_x =(\cnN_x, \cnPUO_x, \cnP_x, \cnCS_x, \cnPUS_x, \cnCUS_x, \cnPS_x) \in \cnANs\}\\
\end{array}$
\end{definition}

Definition~\ref{d:transFuncUpdatePre} establishes the updating of the state transition function for the node. This learning phase selects a new state transition function based on the previous transition function, as well as the previous and currently observed states and the actions (and context) that lead to the new observed state.

\begin{definition}
\label{d:utilityUpdatePre}
Let $\cnAH\!=\!(\cnH, \cnANs)$ be an active cognitive hierarchy with $\cnH\!=\!(\cnNs, \cnN_0, \cnFPs)$. The utility update of $\cnAH$ with respect to an active cognitive node $\cnAN_i\!=\!(\cnN_i,  \cnPUO_i, \cnP_i, \cnCS_i, \cnPUS_i, \cnCUS_i, \cnPS_i) \in \cnANs$, written as $\cnUtilUpdate(\cnAH,\cnAN_i)$ is $\cnAH$, if $i = 0$ else is an active cognitive hierarchy $\cnAH' = (\cnH, \cnANs')$ where $\cnANs' = \cnANs\!\setminus\!\{\cnAN_i\}
\cup \{\cnAN_i'\}$ and $\cnAN'_i = (\cnN_i, \cnPUO_i, \cnP_i, \cnCS_i, \cnPUS_i, \cnCUS_i, \cnUUF_i(\cnPS_i, U))$ s.t:\\
$\begin{array}[t]{l@{}l@{}l@{}l@{}l}
   &       U &~=~ \bigcup  \{ & \cnUTF_{i,x}(\cnPS_x)~|~\langle \cnSF_{i,x}, \cnCF_{x,i}, \cnUTF_{i,x},\cnTF_{x,i}  \rangle\in\cnFPs \mbox{~for~}\\
   &         &             & ~\cnAN_x =(\cnN_x, \cnPUO_x, \cnP_x, \cnCS_x, \cnPUS_x, \cnCUS_x, \cnPS_x) \in \cnANs\}\\
\end{array}$
\end{definition}

The utility update function (Definition~\ref{d:utilityUpdatePre}) updates the internal planning state of a node with the utilities derived from the planning states of its children.

\begin{definition}
\label{d:actionUpdatePre}
Let $\cnAH\!=\!(\cnH, \cnANs)$ be an active cognitive hierarchy with $\cnH\!=\!(\cnNs, \cnN_0, \cnFPs)$. The action update of $\cnAH$ with respect to an active cognitive node $\cnAN_i\!=\!(\cnN_i, \cnPUO_i, \cnP_i, \cnCS_i, \cnPUS_i, \cnCUS_i, \cnPS_i) \in \cnANs$, written as $\cnActionUpdate(\cnAH,\cnAN_i)$ is an active cognitive hierarchy $\cnAH' = (\cnH, \cnANs')$ where
$\cnANs' = \cnANs\!\setminus\!\{\cnAN_i\} \cup \{\cnAN_i'\}$ and $\cnAN'_i = (\cnN_i,  \cnPUO_i, \cnP'_i, \cnCUS_i, \cnPUS_i, \cnCUS_i, \cnPS'_i)$ s.t:\\
$\begin{array}[t]{l@{}l@{}l@{}l@{}l}
  \multicolumn{4}{l}{\langle \cnP'_i, \cnPS'_i \rangle = \cnPO(\cnP_i, \cnPUO_i, T, \cnPS_i, \cnCUS_i),}\\
  T &~= \bigcup\{ & \cnTF_{x,i}(\cnP_x(\cnCS_x))|\langle \cnSF_{i,x}, \cnCF_{x,i}, \cnUTF_{i,x}, \cnTF_{x,i} \rangle\in\cnFPs \mbox{~for~}\\
    &              & ~\cnAN_x =(\cnN_x, \cnPUO_x, \cnP_x, \cnCS_x, \cnPUS_x, \cnCUS_x, \cnPS_x) \in \cnANs\}\\
\end{array}$
\end{definition}

Finally, the action update of a node encodes the process of the planner selecting a new policy. It does so based on the current policy, the current transition function, the task parameters, and the current belief state. A secondary task of this process is to ensure that the node's current belief state is updated to it's newly observed belief state.

Note, when this process is applied to the external world node, $\cnN_0$, the intuition is that the task parameters of the directly connected nodes, which were generated by applying the policy of those nodes to their respective belief states, represents the robot acting on the real world. These task parameters therefore represent physical actions, for example, applying some voltage to an actuator for some period of time.

Now, in order to formalise the updating of a complete hierarchy, we first introduce an intermediate function that successively applies a given function to a sequence of nodes.

\begin{definition}
\label{d:updatePass}
Let $\Phi$ be a function that takes as parameters an active cognitive hierarchy and an active cognitive node within that hierarchy, and returns another active cognitive hierarchy. Furthermore, let $\cnAH = (\cnH, \cnANs)$ be an active cognitive hierarchy, and let ${\vec{Q}}= [\cnAN_1,\ldots, \cnAN_n]$ consist of a sequence over all the  active cognitive nodes in $\cnANs$. Then the update pass of $\Phi$ over $\cnAH$ and $\vec{Q}$, written $\cnUpdatePass(\Phi, \cnAH, \vec{Q})$, is an active cognitive hierarchy:
\[
\cnAH' = \Phi(\ldots\Phi(\cnAH,\cnAN_1),\ldots\cnAN_n)
\]
\end{definition}

Definition~\ref{d:updatePass} provides a mechanism for applying the individual node update functions over the entire cognitive hierarchy in some order. We use this to define the individual passes as part of a complete update of the cognitive hierarchy.

\newcommand{\cnAHT}{\cnAH^{\tinymath{T}}}
\begin{definition}
\label{d:update}
Let $\cnAH = (\cnH, \cnANs)$ be an active cognitive hierarchy with $\Phi$ and $\Psi$ being the upward and downward graphs of $\cnH$ respectively. The process update of $\cnAH$, written $\cnUpdate(\cnAH)$, is an active cognitive hierarchy:
\[
\cnAH' = \cnActionUpdateP(\cnUtilUpdateP(\cnTransFuncUpdateP(\cnCorrUpdateP(\cnPredUpdateP(\cnAH)))))
\]
where:
\begin{itemize}
\item $\cnPredUpdateP(\cnAHT) = \cnUpdatePass(\cnPredUpdate, \cnAHT,\vec{Q}_D)$
\item $\cnCorrUpdateP(\cnAHT) = \cnUpdatePass(\cnCorrUpdate, \cnAHT,\vec{Q}_U)$
\item $\cnTransFuncUpdateP(\cnAHT) = \cnUpdatePass(\cnTransFuncUpdate, \cnAHT,\vec{Q}_U)$
\item $\cnUtilUpdateP(\cnAHT) = \cnUpdatePass(\cnUtilUpdate, \cnAHT,\vec{Q}_U)$
\item $\cnActionUpdateP(\cnAHT) = \cnUpdatePass(\cnActionUpdate, \cnAHT,\vec{Q}_D)$
\item $\vec{Q}_U$ is a sequence consisting of all active cognitive nodes in $\cnANs$ such that the sequence satisfies the partial ordering induced by the upward graph $\Phi$.
\item $\vec{Q}_D$ is a sequence consisting of all active cognitive nodes in $\cnANs$ such that the sequence satisfies the partial ordering induced by the downward graph $\Psi$.
\end{itemize}
\end{definition}

Definition~\ref{d:update} formalises the process update model outlined in Section~\ref{secCogHiOverview}. The cognitive hierarchy's upward graph induces a sequence of nodes passing up the hierarchy, while the downward graph induces a sequence passing down the hierarchy. Depending on the particular update being performed the appropriate sequence is chosen and the corresponding function is repeatedly applied.

Finally, to complete the formal developments, we establish the main result that underlines the importance of the process update model, namely that it is well-defined.

\begin{theorem}
\label{t:welldefined}
The cognitive process model $\cnUpdate$ is well-defined.
\end{theorem}
\begin{proof}
By inspection on the five individual passes of cognitive hierarchy. The proof that each pass is well-defined is structurally identical to the proof for Lemma 1 from~\citeauthor{clark2016framework}~(\citeyear{clark2016framework}), hinging on the partial-ordering induced by the appropriate DAG.
\end{proof}

Theorem~\ref{t:welldefined} establishes that the process model $\cnUpdate$ does indeed produce an active cognitive hierarchy, and that there is only one possible output given the input. Consequently, the process model performs what one would intuitively require of an operational cognitive hierarchy; updating
nodes up the hierarchy and propagating actions back down the hierarchy, resulting in changes to the physical actuators of a robot.

\section{Discussion}

The previous section established the cognitive hierarchy formalism. We now turn to a discussion of its broader implications. We do so with reference to the motivating example, and its implementation within a simulation environment, to illustrates several of its salient features.

The importance of problem decomposition is well established in AI, from hierarchy learning~\cite{DBLP:conf/icml/Kaelbling93} through to factored planning~\cite{DBLP:conf/ijcai/AmirE03}. Formally, the decomposition of a node into two distinct nodes within a cognitive hierarchy can reduce an $M \times N$ search space to an $M + N$ search space~\cite{DBLP:conf/ausai/RajaratnamHPST16}.

Our example illustrates this point. The complexity of reinforcement learning in grid-worlds is $O(n^2)$, where $n$ is the size of the state-space \cite{DBLP:conf/aaai/KoenigS93}. The decomposition into rooms reduces complexity from order $n^2 = 450^2$ to $5\times 90^2+5^2$. Additionally, the similarity of the five rooms allows transfer learning of the transition and policy functions. Together the decomposition reduces the complexity of learning the optimal policy by a factor of 25. Figure \ref{figLearning} shows the faster convergence in the optimum number of steps to reach the goal for the decomposed problem.

\begin{figure}[h]
  \centering
    \includegraphics[width=0.40\textwidth]{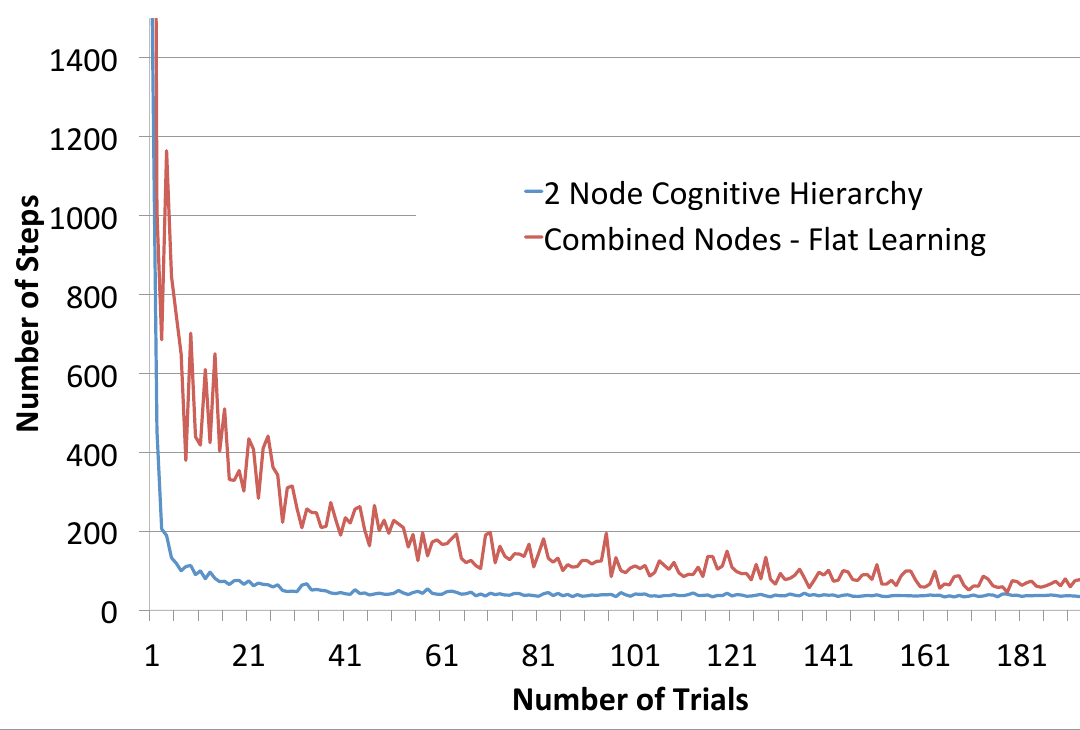}
    \caption{Learning to minimise the steps to reach the goal converges faster if the problem is decomposed (results averaged over 10 runs). }
  \label{figLearning}
\end{figure}

To ensure proper behaviour with a hierarchical decomposition it is necessary to account for action utilities. In our example, a purely abstract inter-room planner, with no access to low-level action costs, might choose to minimise the number of rooms to the goal rather than choosing the shortest path. However, the cognitive hierarchy is able to account for the local intra-room costs with the utility update and choose an optimal path that traverses three instead of two rooms (Figure \ref{figDrunkPaths}, left). Importantly, this behaviour is dependent on the stochasticity of the motion model and when the amount of slippage increases, an optimal path minimises the number of doorways as an increasing number of attempts may be necessary to move between rooms (Figure~\ref{figDrunkPaths}, right).

\begin{figure}[h]
  \centering
    \includegraphics[width=0.40\textwidth]{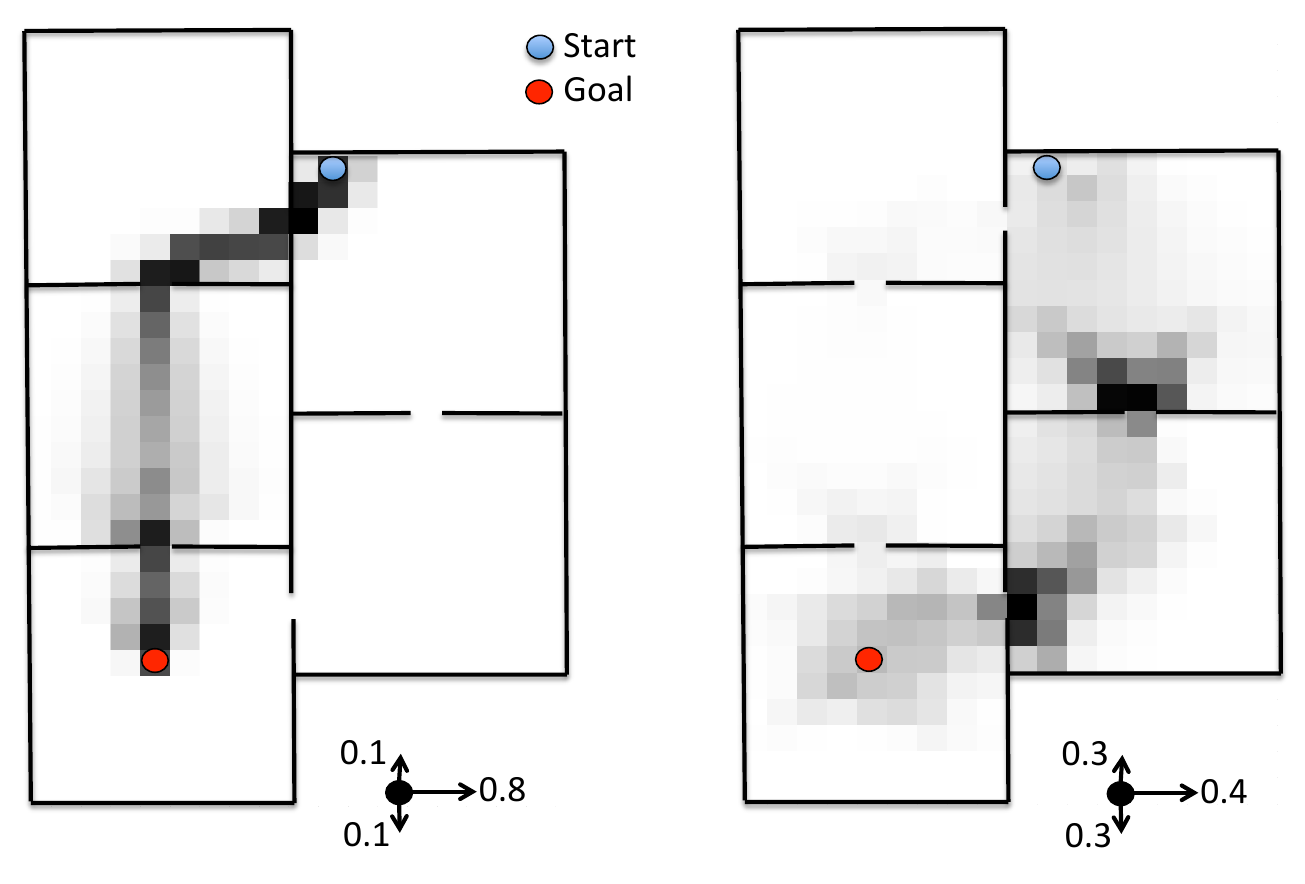}
    \caption{Location visitation count for optimal policy trials ($>$ 1000). The depth of shading indicates the frequency of visits.
      Increasing the slippage to only $40\%$ in the intended direction, and $60\%$ for either side,  changes the optimal to the longer path with fewer doorways.}
  \label{figDrunkPaths}
\end{figure}

In our motivating example the high-level planner was able to find the optimal path to the goal using the costs learnt by the low-level motion model. However, in general, hierarchical reinforcement learning cannot guarantee that a globally optimal policy will be learnt. In particular, a subtask cannot take into account the context in which it is learnt and therefore the best that can be hoped for is \emph{recursive optimality}~\cite{DBLP:journals/jair/Dietterich00}. Recursive optimality can be guaranteed in our mixed representation cognitive hierarchy for the general case, where the lower-level learns without access to the high-level context of the planner.


\section{Conclusion and Future Work}
In this paper we formalised the integration of online learning and planning, using differing representations, within a cognitive hierarchy. This formalisation takes the form of a meta-theory, whereby the interaction between nodes in a hierarchy is specified while leaving open the internal details of each node. This provides the flexibility to use arbitrary representations and reasoning mechanisms within individual nodes. The process model, that defines how to update an operational cognitive hierarchy, was also shown to be well-defined.

The key features of the extended formalisation was to show how planning and online learning can be integrated into a cognitive hierarchy. To this end a motivating example was introduced to illustrate the power and flexibility of the framework. The example established how a reinforcement learner could be used to learn both a transition model and action policy for a robot moving within a room, and how this could be combined with a symbolic planner for navigating between rooms. Costs for low-level navigation were passed to the high-level to ensure that the planner could find a lowest cost plan. 

This work opens up a number of interesting avenues for future research. In the first place it will be important to validate the framework on a variety of robot platforms and problems, from simple robots with limited capabilities through to complex robots such as those used for the Robocup@Home challenge.

\comment{
A second important avenue for future work will be to explore more deeply the interaction of learning and planning at multiple levels of the hierarchy. In the motivating example, the appropriate action utilities of the low-level node were available to the planning node. However, there are cases where only partial costs might be available at any one time. For example, it may be infeasible to generate all costs if there are a large number of features across multiple rooms.

Interestingly, our framework does provides a mechanism which could be adapted to achieve this. A process model update that ignores the external world, $\cnN_0$, could allow the robot to perform a series of what-if scenarios that did not involve sensing or acting in the external world. This would provide a form of highly deliberative ``deep'' reasoning where the robot simulates at multiple levels of abstraction before making a decision.
}

A second important avenue for future work will be to explore more deeply the interaction of learning and planning at multiple levels of the hierarchy. In some cases it may be desirable to perform a series of what-if scenarios involving multiple levels of the hierarchy. Interestingly, our framework could achieve this through a process model that ignores the external world node, $\cnN_0$. The process model would then provide a ``deep'' simulation that could be used for highly deliberative reasoning.


%
\newpage

\bibliographystyle{named}
\bibliography{paper}

\end{document}